\documentclass[letterpaper, 10 pt, conference]{ieeeconf}  

\IEEEoverridecommandlockouts                              
\usepackage{cite}
\usepackage{amsmath}
\usepackage{amssymb,amsfonts,bm}
\usepackage{algorithmic}
\usepackage{graphicx}
\usepackage{textcomp}
\usepackage{mathtools}
\usepackage{booktabs}
\usepackage{multicol}
\usepackage{tikz}

\usepackage{textcomp}
\usepackage{hyperref}
\usepackage{eso-pic}

\usepackage{xcolor}

\newtheorem{theorem}{\bf{Theorem}}
\newtheorem{lemma}{\bf{Lemma}}
\newtheorem{corollary}{\bf{Corollary}}
\newtheorem{assumption}{\bf{Assumption}}
\newtheorem{proposition}{\bf{Proposition}}
\newtheorem{remark}{\bf{Remark}}

\renewenvironment{proof}{{\bf{Proof:}}}{$\hfill\blacksquare$}

\title{\LARGE \bf
Physically Consistent Learning of Conservative Lagrangian Systems with Gaussian Processes
}

\author{Giulio Evangelisti and Sandra Hirche
\thanks{Both authors are members of the Chair of Information-oriented Control, Department of Electrical and Computer Engineering, Technical University of Munich, D-80333 Munich, Germany
        {\tt\small [giulio.evangelisti, hirche]@tum.de}}%
}

\AddToShipoutPicture{\begin{tikzpicture}[remember picture,overlay]
		\node[anchor=north,yshift=-10pt,align=center] at (current page.north) {\footnotesize This is the accepted version of a paper which has been published by IEEE in 2022 IEEE 61st Conference on Decision and Control (CDC).\\ \footnotesize The final published paper can be found at \href{https://doi.org/10.1109/CDC51059.2022.9993123}{DOI: 10.1109/CDC51059.2022.9993123}};
\end{tikzpicture}}

\begin{document}

\newcommand\copyrighttext{%
	\footnotesize \textcopyright 2023 IEEE. Personal use of this material is permitted. Permission
	from IEEE must be obtained for all other uses, in any current or future
	media, including reprinting/republishing this material for advertising or
	promotional purposes, creating new collective works, for resale or
	redistribution to servers or lists, or reuse of any copyrighted
	component of this work in other works.}
\newcommand\copyrightnotice{%
	\begin{tikzpicture}[remember picture,overlay]
		\node[anchor=south,yshift=5pt] at (current page.south) {\fbox{\parbox{\dimexpr\textwidth-\fboxsep-\fboxrule\relax}{\copyrighttext}}};
	\end{tikzpicture}%
}

\maketitle
\copyrightnotice
\thispagestyle{empty}
\pagestyle{empty}

\begin{abstract}
This paper proposes a physically consistent Gaussian Process (GP) enabling the data-driven modelling of uncertain Lagrangian systems. The function space is tailored according to the energy components of the Lagrangian and the differential equation structure, analytically guaranteeing properties such as energy conservation and quadratic form. The novel formulation of Cholesky decomposed matrix kernels allow the probabilistic preservation of positive definiteness. Only differential input-to-output measurements of the function map are required while Gaussian noise is permitted in torques, velocities, and accelerations. We demonstrate the effectiveness of the approach in numerical simulation. 

\end{abstract}

\section{INTRODUCTION}

Euler-Lagrange (EL) systems represent a wide variety of physical systems and are the foundation of a large class of theoretical concepts in control, such as stability, passivity, and in feedback law design. The application of Gaussian Processes (GPs) in learning-based control holds great promise for the control of uncertain systems, improving important aspects like performance \cite{perf_guaraentees_gp}, safety \cite{berkenkamp}, and data-efficiency~\cite{data_eff_gps}. Yet, in general GP regression does not account for physically consistency, hampering the provision of safety guarantees when used in model-based control approaches.

We address this issue by formulating GPs consistent with the relevant physical laws and mathematical models, thus implicitly performing a trade-off between the flexibility of learning and physics-imposed symmetries. Traditionally, parametric techniques \cite{sys_id_ljung} have been applied to the identification of EL systems, as reviewed in \cite{param_id_leboutet}. However, mostly assuming linearity in the parameters, these methods eventually become infeasible with increasingly nonlinear, and coupled, dynamics, as for example in soft robots \cite{soft_robotics_com}, aircrafts \cite{day_couplings_air}, and marine vehicles~\cite{hydro}. 

Based on Bayesian principles, GPs allow the inclusion of analytical prior knowledge, thus forming a bridge to parametric model-based approaches. The authors of \cite{symp_GPR_ham_sys} propose the regression of Hamiltonian systems with GPs, preserving the symplectic structure, but without any guarantees and requiring the availability of noiseless impulse measurements. A Lagrangian kernel is derived in \cite{cheng_lagrangian_rkhs} within the framework of Reproducing Kernel Hilbert Spaces (RKHS). While analogies to our stochastic GP-based approach naturally exist, the work lacks a further specification of the associated function space, prohibiting a comprehensive analysis of control relevant properties closely intertwined with physical consistency. In \cite{thomas_tracking_control_of_el_systems}, the EL residual dynamics are compensated by employing independent GPs per degree of freedom, neglecting correlations and structure. This leads to an implicit differential equation system, the necessity of noiseless acceleration measurements and the curse of dimensionality. Guaranteeing physical consistency of data-driven methods still represents a major open problem.  

The main contribution of this work is the derivation of a physically consistent Lagrangian-Gaussian Process (L-GP) constructed by linear operators, embedding the differential equation structure into the model. Combined with quadratic kernels, we extend the GP concept from scalar function to  symmetric matrix spaces. Based on a Cholesky decomposition of the underlying matrix kernel, we are able to provide a probabilistic guarantee for the positive definiteness of the estimates. Further physical consistencies of the L-GP analytically preserving equilibria, passivity, and conservatism are formally shown.

After defining the problem setting in Section~\ref{sec_prob_form}, we review the framework of GPs in Section~\ref{gp_frame} and describe our proposed unification with EL systems in Section~\ref{l_gp_sec}. Section~\ref{inv_properties} presents our main results. Numerical illustrations validate theory and effectiveness of the approach in Section~\ref{sim_sec}.

\section{PROBLEM FORMULATION}\label{sec_prob_form}

In this paper, we consider the class of conservative EL systems whose equations of motion are given by\footnote{\textbf{Notation:} Vectors $\bm{a}$ and matrices $\bm{A}$ are denoted with bold lower and upper case characters, respectively. The nabla operator $\nabla$ multiplied with a differentiable scalar field $f(\bm{x})$ gives its gradient w.r.t.\ $\bm{x} \in \mathcal{X}$. The partial gradient w.r.t.\ a subset of variables $\bm{y} \in \mathcal{Y}$, $\mathcal{Y} \subset \mathcal{X}$, is described by $\nabla_{\bm{y}}f$. $\bm{I}$ denotes the identity, $\bm{0}$ the zero and $\bm{1}$ the ones matrix. A collection of $D$ indexed vectors $\bm{x}_i\in\mathbb{R}^N$ is stacked row-wise in the matrix $\bm{X}=[\bm{x}^\top_i]_{i=1}^D\in\mathbb{R}^{D\times N}$ such that $x_{ij}$ refers to the $j$-th component of $\bm{x}_i$. For notational convenience, we omit the indexing in the matrix expressions writing $[x_{ij}]=\bm{X}=[\bm{x}^\top_i]=[\bm{y}_j]=\bm{Y}^\top=[y_{ij}]$. $\mathrm{E}[\cdot]$, $\mathrm{Var}[\cdot]$ and $\mathrm{Cov}[\cdot]$ denote the expectation, variance and covariance operators w.r.t.\ random variables. The Hadamard product and square are denoted by $\bm{x}\circ\bm{y}$ and $\bm{x}^{\circ2}=\bm{x}\circ\bm{x}$, respectively, the Kronecker product with $\otimes$, and the direct sum with $\oplus$. With $\lambda_k(\cdot)$ we denote the $k$-th largest eigenvalue corresponding to the components of the decreasingly ordered eigenvalue vector $\bm{\lambda}^{\downarrow}(\cdot)$.}
\begin{equation} 
\frac{d}{dt} \left(\nabla_{\dot{\bm{q}}} L\right) - \nabla_{\bm{q}} L = \bm{\tau} \label{lagrange_nabla_eqs_of_mot}
\end{equation}
based on the set of generalized coordinates $\bm{q}\in\mathbb{R}^N$ and the unknown Lagrangian $L$. The generalized forces $\bm{\tau}\in\mathbb{R}^N$ are the measurable, possibly underactuated, inputs to the system.

\begin{assumption}\label{ass_autonomous}
The unknown Lagrangian is autonomous with $L\equiv L(\bm{q},\dot{\bm{q}}) = T(\bm{q},\dot{\bm{q}}) - V(\bm{q})$ and composed of the difference between unknown kinetic energy $T\colon\mathbb{R}^N\times\mathbb{R}^N\to\mathbb{R}$ and unknown potential energy $V\colon\mathbb{R}^N\to\mathbb{R}$.
\end{assumption}

\begin{assumption}\label{ass_kinetic}
The kinetic energy is of the form $T(\bm{q},\dot{\bm{q}}) = \frac{1}{2}\dot{\bm{q}}^T\bm{M}(\bm{q})\dot{\bm{q}}$ where $\bm{M}\colon\mathbb{R}^N\to\mathbb{R}^{N\times N}$ is the unknown, (symmetric) positive definite, inertia matrix. 
\end{assumption}

\begin{assumption}\label{ass_potential}
The potential energy is decomposable into the sum $V(\bm{q}) = G(\bm{q}) + U(\bm{q})$ of gravitational and elastic potential energies $G$ and $U$, respectively, where the former has an equilibrium, $G(\bm{0})=0$, $\nabla_{\bm{q}} G(\bm{0})=\bm{0}$, and the latter is given by $U(\bm{q}) = \frac{1}{2}\bm{q}^T\bm{S}(\bm{q})\bm{q}$ with unknown, (symmetric) positive definite, stiffness matrix $\bm{S}\colon\mathbb{R}^N\to\mathbb{R}^{N\times N}$.
\end{assumption}

Assumptions~\ref{ass_autonomous}--\ref{ass_potential} essentially describe the considered system class and represent its physical properties. Assumption~\ref{ass_autonomous} restricts the systems described by \eqref{lagrange_nabla_eqs_of_mot} to be only implicitly dependent on time and free from disturbances caused by other external forces such as dissipation. Extending to handle dissipative and time-variant systems is possible and topic of current research. Assumption~\ref{ass_kinetic} is a valid approach when dealing, e.g., with classical mechanical rigid body systems, while Assumption~\ref{ass_potential} constrains chosen coordinates to have an equilibrium at the origin.

Exploiting Assumptions \ref{ass_autonomous}--\ref{ass_potential} and applying the chain rule to \eqref{lagrange_nabla_eqs_of_mot}, we obtain the well-known matrix-vector expression
\begin{equation} 
\bm{M}(\bm{q})\ddot{\bm{q}} + \bm{C}(\bm{q},\dot{\bm{q}})\dot{\bm{q}} + \bm{g}(\bm{q}) = \bm{\tau}\, , \label{M_C_g_system}
\end{equation}
introducing the (non-unique) generalized Coriolis matrix $\bm{C}(\bm{q},\dot{\bm{q}})\coloneqq\frac{1}{2}[\nabla_{\dot{\bm{q}}}\nabla_{\bm{q}}^\top T + \dot{\bm{M}}(\bm{q}) - \nabla_{\bm{q}}\nabla_{\dot{\bm{q}}}^\top T]$ as in \cite{murray_amitrm} and the vector of generalized potential forces $\bm{g}(\bm{q})\coloneqq\nabla_{\bm{q}}V$, having exploited $\nabla_{\dot{\bm{q}}}\nabla_{\dot{\bm{q}}}^\top T=\bm{M}(\bm{q})$, $\nabla_{\bm{q}}T=\frac{1}{2}(\nabla_{\bm{q}}\nabla_{\dot{\bm{q}}}^\top T)\dot{\bm{q}}$ as well as $\left(\nabla_{\dot{\bm{q}}}\nabla_{\bm{q}}^\top T\right)\dot{\bm{q}}=\dot{\bm{M}}(\bm{q})\dot{\bm{q}}$. Note that there are other ways to define the matrix $\bm{C}(\bm{q},\dot{\bm{q}})$. This particular choice, however, has the unique property that $\dot{\bm{M}}-2\bm{C}$ is skew-symmetric.

Our objective is to approximate the unknown Lagrangian function $L(\bm{q},\dot{\bm{q}})$ by a data-based estimate $\hat{L}(\bm{q},\dot{\bm{q}})$ which is physically consistent, i.e., maximally aligned with the existing physical knowledge, in our case Assumptions~\ref{ass_autonomous}--\ref{ass_potential}. We assume access to noisefree positional observations $\bm{q}_i$, $i = 1,\dots,D$ with $D\in\mathbb{N}$, and to potentially noisy velocity, acceleration and torque measurements
\begin{subequations} 
\begin{align} 
\dot{\bm{x}}_i &= \begin{bmatrix} \dot{\bm{q}}_i \\ \ddot{\bm{q}}_i \end{bmatrix} + \begin{bmatrix} \bm{\omega}_i \\ \bm{\alpha}_i \end{bmatrix}, \!\quad \begin{bmatrix} \bm{\omega}_i \\ \bm{\alpha}_i \end{bmatrix} \sim \mathcal{N}\left(\bm{0}, \begin{bmatrix}  \bm{\Sigma}_{\bm{\omega}_i} \!&\! \bm{0} \\ \bm{0} \!&\! \bm{\Sigma}_{\bm{\alpha}_i} \end{bmatrix}\right) \label{noise_meas_model_dx}, \\
\bm{y}_i &= \bm{\tau}_i(\bm{q}_i,\dot{\bm{q}}_i,\ddot{\bm{q}}_i) + \bm{\varepsilon}_i \, , \quad \bm{\varepsilon}_i \sim \mathcal{N}(\bm{0},\bm{\Sigma}_{\bm{\varepsilon}_i}) \label{noise_meas_model_tau}\, .
\end{align}
\end{subequations}
The noise processes $\{\bm{\omega}_i\}$, $\{\bm{\alpha}_i\}$ and $\{\bm{\varepsilon}_i\}$ are white, zero-mean, uncorrelated, and have known covariance matrices $\bm{\Sigma}_{\bm{\omega}_i}$, $\bm{\Sigma}_{\bm{\alpha}_i}$ and $\bm{\Sigma}_{\bm{\varepsilon}_i}$, respectively. 
After collecting all $D$ observations at positions $\bm{Q}=[\bm{q}_i^\top]$ with analogous matrices $\dot{\bm{X}}=[\dot{\bm{x}}_i^\top]$ and $\bm{Y}=[\bm{y}_i^\top]$, we obtain the training data set $$\mathcal{D}=\{\bm{Q},\dot{\bm{X}},\bm{Y}\}\, .$$ 

\begin{remark}
Extending to noisy positional inputs is possible via certain GP-based methods as in, e.g., \cite{nips_gp_input_noise,johnson_input_noise_gp}, yet they are mostly based on first-order Taylor approximations. 
\end{remark}
 
\section{GAUSSIAN PROCESS FRAMEWORK}\label{gp_frame}

This section gives a brief overview of the mathematical framework of GPs based on \cite{rasmussen_gp_mit_book}. For a more complete introduction, the reader is referred to the literature \cite{raissi_gp_diff_eqs,nips_deriv_obs_solak,md_gps_alvarez}.

\subsection{Inference with GPs}

For $\bm{x},\bm{x}^\prime\in\bm{\mathcal{X}}$ in a continuous domain $\bm{\mathcal{X}}\subseteq\mathbb{R}^{2N}$, a GP with mean $m(\bm{x})$ and covariance $k(\bm{x},\bm{x}^\prime)$, denoted by 
\begin{equation} 
f(\bm{x}) \sim \mathcal{GP}\left(m(\bm{x}), k(\bm{x},\bm{x}^\prime) \right) \label{GP_definition}
\end{equation}
and with scalar $f(\bm{x})\in\mathbb{R}$, is a stochastic process extending the Gaussian probability distribution from random variables to functions. As such, it inherits the convenient mathematical properties of the Normal distribution. In particular, any finite marginalization $\bm{z}=[z_i]$ of function observations $z_i=f(\bm{x_i})+\zeta_i$, corrupted by the zero-mean white-noise process $\zeta_i\in\mathbb{R}$ and corresponding to the finite subset of function evaluations at $\{\bm{x}_i\}\subset\bm{\mathcal{X}}$, is jointly Gaussian distributed.

A GP is fully characterized by its mean and kernel function $m(\bm{x})=\mathrm{E}[f(\bm{x})]$ and $k(\bm{x},\bm{x}^\prime)=\mathrm{Cov}[f(\bm{x}), f(\bm{x}^\prime)]$, respectively. The former allows the analytical inclusion of prior knowledge about the unknown function, whereas the latter determines higher level functional properties such as smoothness. Estimation of the covariance describing hyperparameters is mostly done via optimization of the marginal likelihood to maximize the probability of observing the measured outputs.

For regression, GPs exploit the joint Gaussian distribution of observations $\bm{z}$ at measurement locations $\bm{X}=[\bm{x}_i^\top]$ and prediction $f(\bm{x})$ w.r.t.\ the test input $\bm{x}\in\bm{\mathcal{X}}$ given by
\begin{equation} 
\begin{bmatrix} f(\bm{x}) \\ \bm{z} \end{bmatrix} \sim \mathcal{N}\left(\begin{bmatrix} m(\bm{x}) \\ \bm{m} \end{bmatrix}, \begin{bmatrix} k(\bm{x},\bm{x}) & \bm{k}^\top(\bm{x}) \\ \bm{k}(\bm{x}) & \bm{K}+\bm{\Sigma} \end{bmatrix} \right) \label{joint_gauss_dist} \, ,
\end{equation}
where we have introduced the mean $\bm{m}=[m(\bm{x}_i)]$ and covariance $\bm{k}(\bm{x})=[k(\bm{x}_i,\bm{x})]$ vectors, the Gram matrix $\bm{K}=[k(\bm{x}_i,\bm{x}_j)]$ as well as the noise covariance matrix $\bm{\Sigma}=\mathrm{Var}[\bm{\zeta}]$ with $\bm{\zeta}=[\zeta_i]$. Here, we have dropped the dependencies of $\bm{m}(\bm{X})$, $\bm{k}(\bm{X},\bm{x})$ and $\bm{K}(\bm{X},\bm{X})$ on $\bm{X}$ for notational simplicity. Conditioning on the observations yields a conditional Gaussian distribution with posterior mean $\hat{f}(\bm{x})\equiv\mathrm{E}[f(\bm{x})|\bm{z}]$ given analytically by 
\begin{equation} 
\hat{f}(\bm{x}) = m(\bm{x}) + \bm{k}^\top(\bm{x}) \left( \bm{K}+\bm{\Sigma} \right)^{-1}\left( \bm{z} - \bm{m} \right). \label{post_gp_mean}
\end{equation}

\subsection{Linear Operators and GPs}

Derivatives and integrals of GPs remain GPs, since these operations are linear.  Applying a linear transformation operator $\bm{\mathcal{T}}_{\bm{x}}$ to the GP of \eqref{GP_definition} leads to a new, possibly multidimensional, GP in the form of 
\begin{equation} 
\bm{\mathcal{T}}_{\bm{x}}f(\bm{x}) \sim \bm{\mathcal{GP}}\left(\bm{m}_{\bm{\mathcal{T}}}(\bm{x}), \bm{K}_{\bm{\mathcal{T}}}(\bm{x},\bm{x}^\prime) \right) \label{GP_linear_op}
\end{equation}
with transformed mean $\bm{m}_{\bm{\mathcal{T}}}(\bm{x})=\bm{\mathcal{T}}_{\bm{x}}m(\bm{x})$ and kernel 
\begin{equation} 
\bm{K}_{\bm{\mathcal{T}}}(\bm{x},\bm{x}^\prime) = \bm{\mathcal{T}}_{\bm{x}} k(\bm{x},\bm{x}^\prime) \bm{\mathcal{T}}^\top_{\bm{x}^\prime} \, . \label{lin_op_app_to_kernel}
\end{equation}
Here, $\bm{\mathcal{T}}^\top_{\bm{x}^\prime}$ is applied from the right \cite{symp_GPR_ham_sys} as is necessary when considering matrix operators.

\section{REGRESSION OF LAGRANGIAN SYSTEMS}\label{l_gp_sec}

Our key idea is to consistently model the unknown Lagrangian using a GP. 
We propose a structural approach for the application of GP regression \eqref{GP_definition}--\eqref{post_gp_mean} to EL systems \eqref{lagrange_nabla_eqs_of_mot}--\eqref{M_C_g_system}: deriving specific GPs and kernels, we ensure consistency with physical requirements, as in Assumptions \ref{ass_autonomous}--\ref{ass_potential}, and embed the structure of the differential equations using Nabla operators~$\nabla$, thereby inducing multidimensional L-GPs according to \eqref{GP_linear_op}--\eqref{lin_op_app_to_kernel} based on the linearity of differentiation.

\subsection{Energy Structuring}
Firstly, we model the Lagrangian as a composite GP $L(\bm{q},\dot{\bm{q}}) = T(\bm{q},\dot{\bm{q}}) - V(\bm{q})$ writing
\begin{equation} 
L(\bm{q},\dot{\bm{q}}) \sim \mathcal{GP}\left(m_L(\bm{q},\dot{\bm{q}}),k_L(\bm{q},\dot{\bm{q}},\bm{q}^\prime,\dot{\bm{q}}^\prime)\right) \label{l_gp_model}
\end{equation}
with the subjacent kinetic and potential energy GPs
\begin{subequations} 
\begin{align} 
T(\bm{q},\dot{\bm{q}}) &\sim \mathcal{GP}\left(m_T(\bm{q},\dot{\bm{q}}),k_T(\bm{q},\dot{\bm{q}},\bm{q}^\prime,\dot{\bm{q}}^\prime)\right)\, , \label{kinetic_GP} \\
V(\bm{q}) &\sim \mathcal{GP}\left(m_V(\bm{q}),k_V(\bm{q},\bm{q}^\prime)\right) \, . \label{potential_GP}
\end{align}
\end{subequations}
Here, the means $m_L=m_T-m_V$ and covariances $k_L=k_T+k_V$ specify the L-GP \eqref{l_gp_model} in a structurally consistent manner. In particular, the potential energy GP \eqref{potential_GP} depends only on the position $\bm{q}$ by construction. It is split up further into the sum $V(\bm{q})=G(\bm{q})+U(\bm{q})$ of independent gravitational and elastic potential GPs\begin{subequations} 
\begin{align} 
G(\bm{q}) &\sim \mathcal{GP}\left(m_G(\bm{q}),k_G(\bm{q},\bm{q}^\prime)\right) \, , \label{grav_pot_GP} \\
U(\bm{q}) &\sim \mathcal{GP}\left(m_U(\bm{q}),k_U(\bm{q},\bm{q}^\prime)\right) \, . \label{elast_pot_GP}
\end{align}
\end{subequations}

\subsection{Quadratic and Positive Definite Structures}\label{ssec_quad_pdf_strucs}

Consistency of kinetic and elastic GPs is ensured by according design of their mean and covariance functions. 
If available, we impose on the according priors the same quadratic structure based on, e.g., for the kinetic energy, the (symmetric) positive definite a-priori matrix $\bm{M}_0(\bm{q})\succ\bm{0}$, $\forall\bm{q}\in\mathbb{R}^N$, as described in the following. 

\begin{assumption}
The priors $m_{f_k}\in\{m_U,m_T\}$ are in the same form as the unknown energy functions
\begin{equation}
f_k(\bm{q}) \coloneqq \frac{1}{2} \overset{(k)}{\bm{q}}\vspace{-0.2cm}\hspace{-0.025cm}^{\top}\vspace{0.2cm} \bm{H}_k(\bm{q}) \overset{(k)}{\bm{q}} \label{quadratic_forms}
\end{equation}
for $k\in\{0,1\}$, where $m_{f_0}\equiv m_U$ with $\bm{H}_{00}\equiv\bm{S}_0$ and $m_{f_1}\equiv m_T$ with $\bm{H}_{10}\equiv\bm{M}_0$. In addition, the gravitational prior fulfills the equilibrium condition such that $m_G(\bm{0})=0$ and $\nabla_{\bm{q}} m_G(\bm{0})=\bm{0}$ hold.
\end{assumption}

Having specified the means in \eqref{kinetic_GP} and \eqref{elast_pot_GP}, we now move on to the kernels. Consider the quadratic functional
\begin{equation} 
\kappa_k(\bm{q},\bm{q}^\prime) = \frac{1}{4} \Big(\overset{(k)}{\bm{q}} \circ \overset{(k)}{\bm{q}}\vspace{-0.2cm}^{\prime}\vspace{0.2cm} \Big)^\top \bm{\Theta}_k(\bm{q},\bm{q}^\prime) \Big(\overset{(k)}{\bm{q}} \circ \overset{(k)}{\bm{q}}\vspace{-0.2cm}^{\prime}\vspace{0.2cm}\Big) \label{quadratic_kernel_functional}
\end{equation} 
with differential-operational index $k\in\{0,1\}$ and the (symmetric) positive definite Cholesky decomposed matrix kernel
\begin{equation} 
\bm{\Theta}_k(\bm{q},\bm{q}^\prime) = \bm{R}_k^\top(\bm{q},\bm{q}^\prime)\bm{R}_k(\bm{q},\bm{q}^\prime) \label{theta_matrix} \, .
\end{equation}
The Cholesky factor $\bm{R}_k$ is an upper-right triangular matrix
\begin{equation} 
\bm{R}_k(\bm{q},\bm{q}^\prime) = \begin{bmatrix} r_{k11}(\bm{q},\bm{q}^\prime) & \hdots & r_{k1N}(\bm{q},\bm{q}^\prime) \\  & \ddots & \vdots \\ \bm{0} &  & r_{kNN}(\bm{q},\bm{q}^\prime) \end{bmatrix} \label{cholesky_factor_R_mat}
\end{equation}
consisting for $0\leq n\leq m\leq N$ of kernels $r_{knm}(\bm{q},\bm{q}^\prime)$. We now use \eqref{quadratic_kernel_functional} to set the covariances in \eqref{kinetic_GP} and \eqref{elast_pot_GP} to:\begin{subequations}
\begin{align}
k_U &\equiv k_{f_0}\coloneqq\kappa_0 \label{pot_en_kernel}\, , \\
k_T &\equiv k_{f_1}\coloneqq\kappa_1 \label{kin_en_kernel} \, .
\end{align}
\end{subequations}
Thus, we implicitly model the stiffness and inertia matrices denoted in accordance with \eqref{quadratic_forms} by $\bm{H}_k\coloneqq[h_{knm}]$ as symmetric, matrix-valued, GPs with entries
\begin{equation} 
h_{knm}(\bm{q})\sim\mathcal{GP}\left(\eta_{knm}(\bm{q}),\theta_{knm}(\bm{q},\bm{q}^\prime)\right) \, , \nonumber
\end{equation}
where $h_{knm}=h_{kmn}$, and analagously constructing $\bm{H}_{k0}\coloneqq[\eta_{knm}]$. The kernels $\theta_{knm}=\bm{\rho}_{kn}^\top\bm{\rho}_{km}$ stem from the inner product of the non-zero columns from \eqref{cholesky_factor_R_mat}, i.e, $\bm{\rho}_{km}\coloneqq[r_{klm}]\in\mathbb{R}^{P}$ for $P=\min\{n,m\}$. Note that, due to the nonlinearity of multiplication, the $r_{klm}$ may not be associated any further with underlying GPs, or more specifically, Gaussian Random Variables (GRVs). However, the covariances $\theta_{knm}$ remain valid kernel functions according to \cite{rasmussen_gp_mit_book}. 


\subsection{Differential Equation Embedding} Let us now use the chain rule to expand $\bm{\mathcal{L}}_{\bm{q}} \equiv \frac{d}{dt}\nabla_{\dot{\bm{q}}}-\nabla_{\bm{q}}$ from \eqref{lagrange_nabla_eqs_of_mot} by writing
\begin{equation} 
\bm{\mathcal{L}}_{\bm{q}} = \left(\nabla_{\dot{\bm{q}}}^\top\ddot{\bm{q}}+\nabla_{\bm{q}}^\top \dot{\bm{q}}\right)\nabla_{\dot{\bm{q}}} - \nabla_{\bm{q}} \coloneqq \bm{\mathcal{L}}_{\bm{q},\dot{\bm{q}}}(\dot{\bm{q}},\ddot{\bm{q}}) \, . \label{lagrange_op}
\end{equation}
Applying the Lagrangian-differential vector operator \eqref{lagrange_op} to the GP from \eqref{l_gp_model}, we obtain for the torques \eqref{lagrange_nabla_eqs_of_mot}--\eqref{M_C_g_system} a vector-valued GP $\bm{\tau}\equiv\bm{\tau}(\bm{q},\bm{q}^\prime;\dot{\bm{q}},\ddot{\bm{q}},\dot{\bm{q}}^\prime,\ddot{\bm{q}}^\prime)\in\mathbb{R}^N$ denoted by
\begin{equation} 
\bm{\tau} \sim \bm{\mathcal{GP}}\left(\bm{m}_{\bm{\tau}}(\bm{q};\dot{\bm{q}},\ddot{\bm{q}}), \bm{K}_{\bm{\tau}}(\bm{q},\bm{q}^\prime;\dot{\bm{q}},\ddot{\bm{q}},\dot{\bm{q}}^\prime,\ddot{\bm{q}}^\prime)\right) \label{tau_vec_l_gp}
\end{equation}
with mean vector $\bm{m}_{\bm{\tau}}=\bm{\mathcal{L}}_{\bm{q}} m_L$ and kernel matrix $\bm{K}_{\bm{\tau}}=\bm{\mathcal{L}}_{\bm{q}} \bm{\mathcal{L}}_{\bm{q}^\prime}^\top k_L$. As indicated in \eqref{tau_vec_l_gp} via the semicolon, $\dot{\bm{q}}$ and $\ddot{\bm{q}}$ have the role of regressors, whereas $\bm{q}$ remains a conventional input. 

\subsection{Joint GP Distribution} Finally, we assert a joint GP for the energies \eqref{kinetic_GP}--\eqref{potential_GP} and torques \eqref{tau_vec_l_gp}. Including for $\bm{\gamma}^\top(\bm{q}) \coloneqq [V(\bm{q}), \nabla_{\bm{q}}^\top V(\bm{q})]$ the equilibrium condition $$\bm{\gamma}_{\bm{0}}\coloneqq\bm{\gamma}(\bm{0})=\bm{0}$$ in alignment with Assumption \ref{ass_potential}, we write
\begin{equation} 
\begin{bmatrix} T \\ V \\ \bm{\gamma_0} \\ \bm{y} \end{bmatrix}\! \sim \mathcal{N}\left(\begin{bmatrix} m_T \\ m_V \\ \bm{0} \\ \bm{m}_{\bm{y}} \end{bmatrix}, \begin{bmatrix} \sigma_T^2 & 0 & \bm{0}^{\top} & \bm{k}_{\bm{y}T}^\top \\ 0 & \sigma_V^2 & \bm{k}_{\bm{0}V}^{\top} & \bm{k}_{\bm{y}V}^\top \\ \bm{0} & \bm{k}_{\bm{0}V} & \bm{K}_{\bm{0}} & \bm{K}_{\bm{y}\bm{0}}^\top \\ \bm{k}_{\bm{y}T} & \bm{k}_{\bm{y}V} & \bm{K}_{\bm{y}\bm{0}} & \bm{K}_{\bm{y}} \end{bmatrix} \right) \label{gauss_dist}
\end{equation}
dropping dependencies on all variables for notational brevity. Also, we have introduced: the stacked vector of outputs $\bm{y}=\mathrm{vec}(\bm{Y}^\top)$ with mean $\bm{m}_{\bm{y}}=[\bm{m}_{\bm{\tau}}(\bm{q}_i;\dot{\bm{q}}_i,\ddot{\bm{q}}_i)]$ and covariance
\begin{equation} 
\bm{K}_{\bm{y}}=\begin{bmatrix} \bm{K}_{\bm{\tau}}(\bm{q}_i,\bm{q}_j;\dot{\bm{q}}_i,\ddot{\bm{q}}_i,\dot{\bm{q}}_j,\ddot{\bm{q}}_j) \end{bmatrix}+\oplus_{i}\bm{\Sigma}_{\bm{\varepsilon}_i} \label{K_y_eq} \, ,
\end{equation}
the energy variances $\sigma_T^2(\bm{q},\dot{\bm{q}})=k_T(\bm{q},\dot{\bm{q}},\bm{q},\dot{\bm{q}})$ and $\sigma_V^2(\bm{q})=k_V(\bm{q},\bm{q})$, the equilibrium variance
\begin{equation} 
\bm{K}_{\bm{0}} = \begin{bmatrix} k_V(\bm{0},\bm{0}) & \nabla_{\bm{q}^\prime}^\top k_V(\bm{0},\bm{0}) \\ \nabla_{\bm{q}}k_V(\bm{0},\bm{0}) & \nabla_{\bm{q}}\nabla_{\bm{q}^\prime}^\top k_V(\bm{0},\bm{0}) \end{bmatrix} \nonumber
\end{equation}
and covariance $\bm{K}_{\bm{y}\bm{0}} = [\bm{k}_{\bm{y}V}(\bm{0}), \nabla_{\bm{q}}^\top\bm{k}_{\bm{y}V}(\bm{0})]$, as well as the remaining equilibrium-potential $\bm{k}_{\bm{0}V}(\bm{q})$ and Lagrangian-differential energy $\bm{k}_{\bm{y}T}(\bm{q},\dot{\bm{q}})$, $\bm{k}_{\bm{y}V}(\bm{q})$ covariances
\begin{subequations} 
\begin{align}
\bm{k}_{\bm{0}V}(\bm{q}) &= \begin{bmatrix} k_V(\bm{0},\bm{q}) \\ \nabla_{\bm{q}_i} k_V(\bm{0},\bm{q}) \end{bmatrix} \, , \\
\bm{k}_{\bm{y}T}(\bm{q},\dot{\bm{q}}) &= \begin{bmatrix} \bm{\mathcal{L}}_{\bm{q}_i,\dot{\bm{q}}_i}(\dot{\bm{q}}_i,\ddot{\bm{q}}_i)k_T(\bm{q}_i,\dot{\bm{q}}_i,\bm{q},\dot{\bm{q}}) \end{bmatrix} \, , \label{covariance_kyT}\\
\bm{k}_{\bm{y}V}(\bm{q}) &= \left[\nabla_{\bm{q}_i}k_V(\bm{q}_i,\bm{q})\right] \, . \label{covariance_kyV}
\end{align}
\end{subequations}
The fully formulated L-GP \eqref{gauss_dist} is the first part of our contribution. Before moving on to analyze its physical consistency, we propose a method for partial input noise compensation.

\subsection{Nonlinear Noise Compensation} The regressor structure stemming from the application of \eqref{lagrange_op} to \eqref{kin_en_kernel} allows for the compensation of noise in the differential inputs $\dot{\bm{q}}_i$ and $\ddot{\bm{q}}_i$. Therefore, as a last step, we combine \eqref{M_C_g_system} and \eqref{tau_vec_l_gp} in order to transform \eqref{noise_meas_model_tau} into 
\begin{equation} 
\bm{y}_i = \bm{\tau}(\bm{q}_i;\dot{\bm{x}}_i) + \tilde{\bm{\varepsilon}}_i \nonumber 
\end{equation}
with the transformed composite noise variable
\begin{equation} 
\tilde{\bm{\varepsilon}}_i= -(\nabla_{\dot{\bm{q}}_i}\nabla_{\dot{\bm{q}}_i}^\top T)\bm{\alpha}_i -(\nabla_{\dot{\bm{q}}_i}\nabla_{\bm{q}_i}^\top T)\bm{\omega}_i + \bm{\varepsilon}_i \nonumber \, .
\end{equation}
Approximating the product of two GRVs to remain Gaussian, we obtain a heteroscedastic (state-dependent) noise model
\begin{equation} 
\tilde{\bm{\varepsilon}}_i \sim \mathcal{N}\left(\bm{0}, \bm{\Sigma}_{\tilde{\bm{\varepsilon}}_i}(\bm{q}_i,\dot{\bm{q}}_i)\right) \label{pi_gnm}\, ,
\end{equation} 
where $\bm{\Sigma}_{\tilde{\bm{\varepsilon}}_i}=\bm{\Sigma}_{\tilde{\bm{\alpha}}_i}(\bm{q}_i)+\bm{\Sigma}_{\tilde{\bm{\omega}}_i}(\bm{q}_i,\dot{\bm{q}}_i)+\bm{\Sigma}_{\bm{\varepsilon}_i}$. After extensive computations, it can be verified that the transformed $\bm{\Sigma}_{\tilde{\bm{\alpha}}_i}\coloneqq\mathrm{Var}[(\nabla_{\dot{\bm{q}}_i}\nabla_{\dot{\bm{q}}_i}^\top T)\bm{\alpha}_i]$ and $\bm{\Sigma}_{\tilde{\bm{\omega}}_i}\coloneqq\mathrm{Var}[(\nabla_{\dot{\bm{q}}_i}\nabla_{\bm{q}_i}^\top T)\bm{\omega}_i]$, exploiting stochastic independencies, are given by 
\begin{align} 
\bm{\Sigma}_{\tilde{\bm{\alpha}}_i} &= \bm{M}_0 \bm{\Sigma}_{\bm{\alpha}_i} \bm{M}_0 + (\bm{1}\!-\!\bm{I})\! \circ\! \bm{\Sigma}_{\bm{\alpha}_i}\! \circ\! \bm{\Theta}_{1ii} + \mathrm{diag}(\bm{\Theta}_{1ii} \bm{\sigma}_{\bm{\alpha}_i})  \nonumber \\
\bm{\Sigma}_{\tilde{\bm{\omega}}_i} &= \bm{C}_0 \bm{\Sigma}_{\bm{\omega}_i} \bm{C}_0^\top + \mathrm{diag}\left( \bm{\Gamma}(\bm{q}_i,\dot{\bm{q}}_i) \bm{\sigma}_{\bm{\omega}_i} \right)
\label{approx_var_mats}
\end{align}
with the inertial variance $\bm{\Theta}_{1ii}\coloneqq\bm{\Theta}_1(\bm{q}_i,\bm{q}_i)$, the Coriolis mean $\bm{C}_0(\bm{q}_i,\dot{\bm{q}}_i)=[\bm{M}_0(\bm{q}_i)\dot{\bm{q}}_i]\nabla_{\bm{q}_i}^\top$ and variance $\bm{\Gamma}(\bm{q}_i,\dot{\bm{q}}_i)=[\bm{\Theta}_1(\bm{q}_i,\bm{q}_i)\dot{\bm{q}}_i^{\circ2}](\nabla_{\bm{q}_i}^{\circ2})^\top$ matrices as well as the main diagonals vectors $\bm{\sigma}_{\bm{\alpha}_i}$ and $\bm{\sigma}_{\bm{\omega}_i}$ of $\bm{\Sigma}_{\bm{\alpha}_i}$ and $\bm{\Sigma}_{\bm{\omega}_i}$, respectively.
 
\begin{proposition}
The distribution \eqref{pi_gnm} with state-dependent variances \eqref{approx_var_mats} is a stochastically consistent approximation which becomes exact for the limit case $\bm{\Sigma}_{\bm{\alpha}_i},\bm{\Sigma}_{\bm{\omega}_i}\to\bm{0}$, or the case of certainly known $T=m_T$ with $k_T\equiv0$. 
\end{proposition}

\begin{proof}
The limit case $\bm{\Sigma}_{\bm{\alpha}_i},\bm{\Sigma}_{\bm{\omega}_i}\to\bm{0}$ leads to noisefree inputs with $\tilde{\bm{\varepsilon}}_i=\bm{\varepsilon}_i \sim \mathcal{N}(\bm{0}, \bm{\Sigma}_{\bm{\varepsilon}_i})$. The case of known $T$ with certainty $k_T\equiv 0$ leads to $\bm{M}_0\equiv\bm{M}$ and $\bm{\Sigma}_{\tilde{\bm{\alpha}}_i} = \bm{M}_0 \bm{\Sigma}_{\bm{\alpha}_i} \bm{M}_0^\top$, respectively. Since the transformation $\tilde{\bm{\alpha}}_i=-\bm{M}\bm{\alpha}$ is linear, $\tilde{\bm{\alpha}}_i\sim\mathcal{N}(\bm{0},\bm{\Sigma}_{\tilde{\bm{\alpha}}_i})$ holds exactly. Proceeding analgously for $\tilde{\bm{\omega}}_i=-(\nabla_{\dot{\bm{q}}_i}\nabla_{\bm{q}_i}^\top T)\bm{\omega}$, we arrive at \eqref{pi_gnm}.
\end{proof}

By casting from the temporal into the spatial domain, \eqref{pi_gnm} unifies the L-GP framework with techniques from nonlinear Kalman filtering \cite{kf_gp_rel}. For effective noise compensation in the inputs $\dot{\bm{q}}_i$ and $\ddot{\bm{q}}_i$, we replace the block matrices $\bm{\Sigma}_{\bm{\varepsilon}_i}$ in \eqref{K_y_eq} by $\bm{\Sigma}_{\tilde{\bm{\varepsilon}}_i}(\bm{q}_i,\bm{\xi}_i)$ from \eqref{pi_gnm}, where $\bm{\xi}_i\coloneqq[\dot{x}_{in}]$. Also, we compute the torque covariance matrix \eqref{K_y_eq} and the Lagrangian-differential kinetic covariance vector \eqref{covariance_kyT} by setting $\dot{\bm{x}}_i^\top=[\dot{\bm{q}}_i^\top,\ddot{\bm{q}}_i^\top]$, $\forall i,j\in\{1,\dots,D\}$.

\section{INVARIANCE PROPERTIES}\label{inv_properties}

Having introduced the novel framework of L-GPs, we now investigate their properties and derive certain guarantees.

\subsection{Quadratic Form}

Employing the kernel structure \eqref{quadratic_kernel_functional}, quadratic form of the energy GPs \eqref{kinetic_GP} and \eqref{elast_pot_GP} can be guaranteed analytically. In order to specify our result in Lemma~\ref{quad_form_lemma}, we require the following definition: for $\delta\bm{x}_i\in\mathbb{R}^N$ and fixed $\bm{q}$, the directional derivative of $\bm{\Theta}_k(\bm{q}_i,\bm{q})$ along $\delta\bm{x}_i$ at $\bm{q}_i$ gives 
\begin{equation} 
\bm{\Phi}_k(\bm{q}_i,\delta\bm{x}_i,\bm{q})\coloneqq\left[\delta\bm{x}_i^\top\nabla_{\bm{q}_i}\theta_{knm}(\bm{q}_i,\bm{q})\right] \label{grad_mat}
\end{equation}
The matrix $\bm{\Phi}_k$ is the directional matrix-derivative of $\bm{\Theta}_k$ from \eqref{theta_matrix} moving through position $\bm{q}_i$ with velocity $\delta\bm{x}_i$. In particular, note that $$\bm{\Phi}_k(\bm{q}_i(\tau),\dot{\bm{q}}_i(\tau),\bm{q})=\dot{\bm{\Theta}}_k(\bm{q}_i(\tau),\bm{q})$$ holds for fixed $\bm{q}$ by deriving w.r.t.\ the virtual time $\tau$.

\begin{lemma} \label{quad_form_lemma}
The mean estimates  $\hat{f}_k(\bm{q})\equiv\mathrm{E}[f_k(\bm{q})|\bm{y},\bm{\gamma_0}]$ of kinetic and elastic GPs \eqref{kinetic_GP} and \eqref{elast_pot_GP} in the joint model \eqref{gauss_dist} are analytically guaranteed to be of quadratic form:
\begin{equation} 
\hat{f}_k(\bm{q}) = \frac{1}{2} \overset{(k)}{\bm{q}}\vspace{-0.2cm}\hspace{-0.05cm}^{\top}\vspace{0.2cm} \hat{\bm{H}}_k(\bm{q})\overset{(k)}{\bm{q}} \label{quad_kinetic_mean} \, . 
\end{equation}
The posterior matrix estimates $\hat{\bm{H}}_k$ are decomposed as
\begin{equation} 
\hat{\bm{H}}_k(\bm{q}) = \bm{H}_{k0}(\bm{q}) + \frac{1}{2}\sum_{i=1}^D\bm{N}_{ki}(\bm{q}) + \bm{N}_{ki}^\top(\bm{q}) \label{M_est}\, ,
\end{equation}
where the basis matrices $\bm{N}_{ki}$ are given by
\begin{align} 
\begin{split}
\bm{N}_{ki}(\bm{q}) &= \frac{\partial^k}{\partial\tau^k} \left[\delta\bm{x}_i \overset{(k)}{\bm{q}}_{\hspace{-0.1cm}i}\vspace{-0.2cm}\hspace{-0.05cm}^{\top}\vspace{0.2cm}\!(\tau) \circ \bm{\Theta}_k(\bm{q}_i(\tau),\bm{q})\right] \\ 
&\quad + \frac{1}{2} (-1)^k \overset{(k)}{\bm{q}}_{\hspace{-0.1cm}i} \overset{(k)}{\bm{q}}_{\hspace{-0.1cm}i}\vspace{-0.2cm}\hspace{-0.05cm}^{\top}\vspace{0.2cm} \circ \bm{\Phi}_k(\bm{q}_i,\delta\bm{x}_i,\bm{q})
\end{split} \label{N_i_matrix}
\end{align}
with the subvectors $\delta\bm{x}_i = [\Delta x_n]_{n=1+(i-1)N}^{iN}$ of the transformed innovation $\Delta\bm{x} = \bm{K}_D^{-1}\Delta\bm{y}$, the Kalman-like kernel gain matrix $\bm{K}_{D}=\bm{K_y}-\bm{K}_{\bm{y0}} \bm{K}_{\bm{0}}^{-1} \bm{K}_{\bm{y0}}^\top$, and the innovation difference $\Delta\bm{y}=\bm{y}-\bm{m}_{\bm{y}}$.
\end{lemma}

\begin{proof}
We set $G\equiv0$ w.l.o.g., leading to $V=U\equiv f_0$ and $\bm{k}_{\bm{0}V}=\bm{0}$ due to \eqref{pot_en_kernel} and \eqref{quadratic_kernel_functional}. Marginalizing \eqref{gauss_dist} over the complementary $f_{\bar{k}}$ for $\{\bar{k}\}=\{0,1\}\setminus\{k\}$ and conditioning on $\bm{y}$ and $\gamma_{\bm{0}}$ yields another Gaussian distribution with mean 
\begin{equation} 
\hat{f}_k(\bm{q}) = m_{f_k}\!(\bm{q}) + \begin{bmatrix} \bm{0}^\top \!&\! \bm{k}_{\bm{y}f_k}^\top\!(\bm{q}) \end{bmatrix} \begin{bmatrix} \bm{K_0} & \bm{K_{y0}}^\top \\ \bm{K_{y0}} & \bm{K_y} \end{bmatrix}^{-1} \begin{bmatrix} \bm{\gamma_0} \\ \Delta\bm{y} \end{bmatrix} \nonumber \, .
\end{equation}
Then, applying standard inversion formulas for partitioned matrices together with $\bm{\gamma_0}=\bm{0}$, we can write
\begin{equation} 
\hat{f}_k(\bm{q})\! =\! m_{f_k}\!(\bm{q}) + \bm{k}_{\bm{y}f_k}^\top\!(\bm{q})(\bm{K_y}\!-\!\bm{K}_{\bm{y0}} \bm{K}_{\bm{0}}^{-1}\! \bm{K}_{\bm{y0}}^\top)^{-1}\Delta\bm{y} \label{cond_mean_kin_init} \, .
\end{equation}
Defining $\bm{l}_{ki}(\bm{q}) \equiv \bm{\mathcal{L}}_{\bm{q}_i}\kappa_k(\bm{q}_i,\bm{q})$ as the application of \eqref{lagrange_op} to \eqref{quadratic_kernel_functional} such that $\bm{k}_{\bm{y}f_k} = \left[\bm{l}_{ki}\right]$, it can be verified that
\begin{align} 
\bm{l}_{ki}(\bm{q}) &= \frac{1}{2}\frac{\partial^k}{\partial\tau^k}\left[ \overset{(k)}{\bm{q}}\!(t) \overset{(k)}{\bm{q}}_{\hspace{-0.1cm}i}\vspace{-0.2cm}\hspace{-0.05cm}^{\top}\vspace{0.2cm}\!(\tau) \circ \bm{\Theta}_k(\bm{q}_i(\tau),\bm{q}(t)) \right] \overset{(k)}{\bm{q}}\!(t) \nonumber \\
& \quad + \frac{1}{4} (-1)^k \sum_n \sum_m \overset{(k)}{q}_{\hspace{-0.115cm}in}\! \overset{(k)}{q}_{\hspace{-0.115cm}im} \!\overset{(k)}{q}_{\hspace{-0.115cm}n} \! \overset{(k)}{q}_{\hspace{-0.115cm}m} \nabla_{\bm{q}_i} \theta_{knm}(\bm{q}_i,\bm{q})\nonumber 
\end{align}
holds after some differential vector-algebraic computations, where we utilize among other intermediate steps that 
\begin{equation}
	\nabla_{\dot{\bm{q}}_i}\nabla_{\bm{q}_i}^\top \kappa_k(\bm{q}_i,\bm{q}) = \begin{cases} 0\, , & k=0 \\ \frac{1}{2}\dot{\bm{q}}\dot{\bm{q}}^\top \circ \bm{\Phi}_k(\bm{q}_i,\dot{\bm{q}}_i,\bm{q}) \, , & k = 1 \end{cases} \, .\notag
\end{equation}
The full derivation is omitted due to space limitations. Further exploiting the commutativity of the Hadamard product and the definition of the projected gradient matrix $\bm{\Phi}_k$ in \eqref{grad_mat}, we follow that $\delta\bm{x}_i^\top\bm{l}_{ki}(\bm{q}) = \frac{1}{2}\bm{q}^{(k)\top} \bm{N}_{ki}(\bm{q}) \bm{q}^{(k)}$ with $\bm{N}_{ki}$ as in \eqref{N_i_matrix}. Then, eliminating the skew-symmetric part of $\bm{N}_{ki}$ combined with the quadratic priors $m_{f_k}$ from Section~\ref{ssec_quad_pdf_strucs}, we reformulate \eqref{cond_mean_kin_init} using $\bm{k}_{\bm{y}f_k}^\top\bm{K}_D^{-1}\Delta\bm{y}=\sum_i\delta\bm{x}_i^\top\bm{l}_{ki}$ and obtain \eqref{quad_kinetic_mean}--\eqref{M_est}.
\end{proof}

Essentially, Lemma \ref{quad_form_lemma} follows from the quadratic kernel structure \eqref{quadratic_kernel_functional}. The summation \eqref{M_est} represents an extension of matrix decompositions from vector to matrix spaces which are aligned with the Lagrangian structure \eqref{lagrange_nabla_eqs_of_mot}--\eqref{M_C_g_system}.

\subsection{Positive Definiteness} 

We now move on to investigate positive definiteness of the matrix estimates $\hat{\bm{H}}_k$. Before that, however, we impose a minor structural restriction on the used kernel functions.

\begin{assumption} \label{iso_quad_met_class_ass}
The Cholesky matrix kernel components $r_{knm}$ from \eqref{cholesky_factor_R_mat} and the gravitational covariance $k_G$ from \eqref{grav_pot_GP} belong to the class of metric kernels $\mathcal{M}$ given by
\begin{equation} 
\mathcal{M} = \big\{ \bm{\Lambda}\succ\bm{0} \,\big|\, k_{\mathcal{M}}(|\bm{d}|) = \sigma^2 \exp \big( - 1/2 \bm{d}^\top \bm{\Lambda} \bm{d} \big) \big\}\label{k_G_form}
\end{equation}
parametrized by $\sigma_{knm}, \bm{\Lambda}_{knm}$ and $\sigma_{G}, \bm{\Lambda}_{G}$, respectively.
\end{assumption}

Note that class $\mathcal{M}$ in Assumption \ref{iso_quad_met_class_ass} includes a wide variety of covariance functions including the commonly used squared exponential kernel \cite{rasmussen_gp_mit_book}. Therein, kernels $k(\bm{q},\bm{q}^\prime)\equiv k_{\mathcal{M}}(|\bm{d}|)$ are Gaussian radial basis functions. They depend only on the squared length of the difference
\begin{equation}
\bm{d}\coloneqq\bm{q}-\bm{q}^\prime
\end{equation}
in the Riemann space \cite[p.~243]{lovlelock_metrics} defined by metric $\bm{\Lambda}\succ\bm{0}$. The special case of variance $\sigma^{2}=(2\pi)^{-N/2}|\det\bm{\Lambda}^{1/2}|$ leads to i.i.d. $\bm{d}\sim\mathcal{N}(\bm{0},\bm{\Lambda}^{-1})$. 

We can now explicitly state the components $r_{knm}$ of the Cholesky factor $\bm{R}_k$ from \eqref{cholesky_factor_R_mat}. Introducing for $0\leq n\leq m\leq N$ the functions $\tilde{r}_{knm} \in \mathcal{M}$ from Assumption~\ref{iso_quad_met_class_ass} with identical Riemannian hypermetrics $\bm{\Lambda}_{knm}=\bm{\Lambda}_{k}$, we set
\begin{equation}
r_{knm}\equiv\tilde{r}_{knm}\, , \!\quad\! \bm{\Sigma}_{f_k}\coloneqq[\sigma^2_{knm}] \, , \!\quad\! \bm{\Sigma}_{\bm{d}_k}\coloneqq\bm{\Lambda}^{-1}_{k} \label{r_knm_spec}\, .
\end{equation}
Note that the hypervariance $\bm{\Sigma}_{f_k}$ is upper-triangular.

We now make a stochastic analysis of the data set $\mathcal{D}$ gathered via \eqref{noise_meas_model_dx}--\eqref{noise_meas_model_tau}, enabling the probabilistic investigation of positive definiteness preservation. Based on the product ${\bm{p}}_{i}^{(n)}$ and difference $\bm{d}_j$ vectors 
\begin{subequations}
\begin{align}
\overset{(n)}{\bm{p}}_{\hspace{-0.1cm}i} &\coloneqq \delta\bm{x}_i \circ \overset{(n)}{\bm{q}}_{\hspace{-0.1cm}i}\, , \quad \forall n\in\{k,2k\} \, , \\
\bm{d}_j &\coloneqq \bm{q} - \bm{q}_j \, , \quad \forall j=1,\dots,D \label{diff_vec_d_j}\, ,
\end{align}
\end{subequations}
we define for $\bm{\mu}_k, \bm{\nu}_k \in \mathbb{R}^{D}$ the composite random variable
\begin{equation}
\bm{\beta}_k\coloneqq\bm{\mu}_k+\bm{\nu}_k \label{beta_eq_pos}
\end{equation}
for those $i\in\{1,\dots,D\}$ which satisfy component-wise $\forall n\in\{k,2k\}$, $k\in\{0,1\}$, the condition ${\bm{p}}_{i}^{(n)}>\bm{0}$, by
\begin{align} 
\begin{split}
\mu_{ki} &= \Upsilon_k\Big(\overset{(2k)}{\bm{p}}_{\hspace{-0.165cm}i}\hspace{0.05cm}\Big)\det\mathrm{diag}\Big(\bm{\varrho}_k^{\downarrow}\circ \overset{(2k)}{\bm{p}}_{\hspace{-0.165cm}i}\vspace{-0.2cm}\hspace{0.025cm}^{\uparrow}\vspace{0.2cm}\Big) \\
  &\quad+ 2k\dot{\vartheta}_{ki} \begin{cases}
    \Upsilon_k\Big(\overset{(k)}{\bm{p}}_{\hspace{-0.1cm}i}\Big)\det\mathrm{diag}\Big(\bm{\varrho}_k^{\downarrow}\circ \overset{(k)}{\bm{p}}_{\hspace{-0.1cm}i}\vspace{-0.2cm}\hspace{-0.05cm}^{\uparrow}\vspace{0.2cm}\Big) & \dot{\vartheta}_{ki} \geq0 \\
    \varrho_{k1} \max_n \overset{(k)}{p}_{\hspace{-0.115cm}in} & \dot{\vartheta}_{ki}<0
  \end{cases} \\
  \nu_{ki} &= \varphi_{ki} \begin{cases}
  	\Upsilon_k\Big(\overset{(k)}{\bm{q}}_{\hspace{-0.1cm}i}\vspace{-0.2cm}\hspace{-0.05cm}^{\circ2}\vspace{0.2cm}\Big)\det\mathrm{diag}\Big(\bm{\varrho}_k^{\downarrow} \circ \overset{(k)}{\bm{q}}_{\hspace{-0.1cm}i}\vspace{-0.2cm}\hspace{-0.05cm}^{\circ2\uparrow}\vspace{0.2cm}\Big) & \varphi_{ki} \geq0 \\
  	\varrho_{k1} \max_n \overset{(k)}{q}_{\hspace{-0.115cm}in}\vspace{-0.2cm}\hspace{-0.2cm}^{2}\vspace{0.2cm} & \varphi_{ki}<0
  \end{cases}
\end{split} \nonumber
\end{align}
with constant, decreasingly ordered, radial vector $\bm{\varrho}_k^{\downarrow}\in\mathbb{R}^N$, $\bm{\varrho}_k^{\downarrow}>\bm{0}$. Here, the projection $\varphi_{ki}$, the differential angle $\dot{\vartheta}_{ki}$ and the normalizing function $\Upsilon_k(\bm{p}_i)$ are given by  
\begin{subequations}
\begin{align}
\varphi_{ki}(\delta\bm{x}_i) &= (-1)^{k}\bm{d}_i^\top \bm{\Lambda}_k \delta\bm{x}_i   \, , \label{projection_phi_ki}\\
\dot{\vartheta}_{ki}(\bm{q}_i(\tau)) &= \varphi_{ki}(\dot{\bm{q}}_i(\tau)) \, , \\
\Upsilon_k(\bm{p}_i)&=1/\Pi_{j=1}^{N-1}\lambda_j(\bm{\Sigma}_{f_k}^2\mathrm{diag}(\bm{p}_i)) \, .
\end{align}
\end{subequations}
If ${\bm{p}}_{i}^{(n)}>\bm{0}$ does not hold $\forall n\in\{k,2k\}$, we define
\begin{equation}
\beta_{ki}\coloneqq\exp(\lVert\bm{d}_i\rVert_{\bm{\Lambda}_k}^2)\lambda_N(\hat{\bm{H}}_{ki})\, ,\quad \overset{(n)}{\bm{p}}_{\hspace{-0.1cm}i}\ngtr \bm{0} \, , \label{beta_k_neg}
\end{equation}
where $\hat{\bm{H}}_{ki}=\frac{1}{2}(\bm{N}_{ki}+\bm{N}_{ki}^\top)$ from \eqref{M_est}. 
The random variable $\bm{\beta}_k$ from \eqref{beta_eq_pos}--\eqref{beta_k_neg} is the result of nonlinearly transformed GRVs stemming from the data set $\mathcal{D}$. We now formulate the main result. 


\begin{theorem}\label{th1}
Consider the posterior, generalized, matrix-valued L-GP estimate $\hat{\bm{H}}_k(\bm{q})$ from \eqref{M_est} with metric covariances $r_{knm}\in\mathcal{M}$ forming the Cholesky kernel $\bm{\Theta}_k(\bm{q}_i,\bm{q})$ from \eqref{theta_matrix} identically specified by  Riemannian hypermetric $\bm{\Sigma}^{-1}_{\bm{d}_k}$ via Assumption \ref{iso_quad_met_class_ass} and \eqref{r_knm_spec}. Positive definiteness of the generalized matrix estimate is guaranteed with probability
\begin{equation} 
\mathrm{Pr}\{\hat{\bm{H}}_k(\bm{q})\succ\bm{0}\} = 1-\delta_k(\bm{q})\textcolor{blue}{\, ,}  \label{pdf_prob_eq_theorem}
\end{equation}
where the physical inconsistency measure $\delta_k(\bm{q})$ is upper bounded according to
\begin{equation} 
\delta_k(\bm{q}) \leq \mathrm{Pr}\left\{\lambda_{N}(\bm{H}_{k0}) + \exp\left[-\mathrm{vec}(\bm{\Sigma}^{-1}_{\bm{d}_k})^\top\bm{D}(\bm{q})\right]\bm{\beta}_k \leq0\right\} \nonumber
\end{equation}
based on the column-wise Kronecker-constructed squared distance matrix
\begin{equation}
	\bm{D}(\bm{q})=[\bm{d}_j(\bm{q})\otimes\bm{d}_j(\bm{q})] \nonumber
\end{equation}
which depends only on the absolute values of the positional joint distances $|\bm{d}_j(\bm{q})|=|\bm{q}-\bm{q}_j|$.
\end{theorem}

\begin{proof}
Firstly, we express the basis matrices \eqref{N_i_matrix} from Lemma \ref{quad_form_lemma} writing $\bm{\Theta}_k\equiv\tilde{\bm{\Theta}}_k(\bm{d})$ based on Assumption \ref{iso_quad_met_class_ass} and \eqref{r_knm_spec}. Proceeding analogously via $\bm{\Phi}_k\equiv\tilde{\bm{\Phi}}_k(\bm{d},\delta\bm{x}_{i})$, we follow that $$\tilde{\bm{\Phi}}_k(\bm{d}_{i},\delta\bm{x}_{i})=(-1)^k2\varphi_{ki}(\delta\bm{x}_i)\tilde{\bm{\Theta}}_k(\bm{d}_{i})$$ with the projection $\varphi_{ki}$ and the difference vector $\bm{d}_i$ as in \eqref{projection_phi_ki} and \eqref{diff_vec_d_j}, respectively. Then, we can write for each matrix summand $\bm{N}_{ki}$ in \eqref{N_i_matrix} that $$\bm{N}_{ki}=\bar{\bm{N}}_{ki} \circ \tilde{\bm{\Theta}}_{k}(|\bm{d}_i|)\, ,$$ where
\begin{equation} 
\bar{\bm{N}}_{ki} = \sum_{n=0}^k \!\big(2\dot{\vartheta}_{ki}(\bm{q}_i)\big)^{k-n} \delta\bm{x}_i \!\overset{(n+k)}{\bm{q}}_{\hspace{-0.3cm}i}\vspace{-0.2cm}\hspace{0.15cm}^{\top}\vspace{0.2cm} +\varphi_{ki}(\delta\bm{x}_i)\overset{(k)}{\bm{q}}_{\hspace{-0.1cm}i} \overset{(k)}{\bm{q}}_{\hspace{-0.1cm}i}\vspace{-0.2cm}\hspace{-0.05cm}^{\top}\vspace{0.2cm} \label{N_i_matrix_simp} 
\end{equation}
dropping dependencies on $\bm{q}$ for the sake of notational brevity. Consider now the $i$-th summand $\hat{\bm{H}}_{ki}=\frac{1}{2}(\bm{N}_{ki}+\bm{N}_{ki}^\top)$ from \eqref{M_est}. Exploiting $$\tilde{\bm{\Theta}}_k(|\bm{d}_{i}|) = \exp(-\lVert\bm{d}_i\rVert_{\bm{\Lambda}_k}^2) \bm{\Sigma}_{f_k}^2\, ,$$ where $\bm{\Sigma}_{f_k}^2=\bm{\Sigma}_{f_k}^\top\bm{\Sigma}_{f_k}$, and its eigenvalue decomposition $$\bm{\Sigma}_{f_k}^2=\bm{W}_k\mathrm{diag}(\bm{\varrho}_k)\bm{W}_k^\top$$ with decreasingly ordered $\bm{\varrho}_k=\bm{\lambda}^{\downarrow}(\bm{\Sigma}_{f_k}^2)>\bm{0}$, we now leverage the eigenvalue behavior of the sum of Hermitian matrices \cite{knutson2000honeycombs} in combination with Weyl's inequality \cite[Theorem~III.2.1]{bhatia_ma}. Thus, with the Hermitian summation $\hat{\bm{H}}_k=\sum_i\hat{\bm{H}}_{ki}$ as in \eqref{M_est}, we derive $$\lambda_{N}(\hat{\bm{H}}_k) \geq \sum_{i=1}^D\lambda_{N}(\hat{\bm{H}}_{ki})\, .$$ Then, utilizing the product constructions \eqref{N_i_matrix_simp}, we obtain
\begin{equation} 
\lambda_{N}(\hat{\bm{H}}_k) \geq \lambda_{N}(\bm{H}_{k0}) + \exp\left[-\mathrm{vec}(\bm{\Lambda}_k)^\top\!\bm{D}\right]\bm{\beta}_k \label{bound_lambda_deriv}
\end{equation}
by successively exploiting the invariance of eigenvalues w.r.t.\ similarity transforms and by using $\lVert\bm{d}_i\rVert_{\bm{\Lambda}_k}^2=(\bm{d}_i\otimes\bm{d}_i)^\top\mathrm{vec}(\bm{\Lambda}_k)$ as well as the definition of the random vector $\bm{\beta}_k \in \mathbb{R}^{D}$ as given component-wise by \eqref{beta_eq_pos}--\eqref{beta_k_neg}. Taking the complement of the set in which the lower bound \eqref{bound_lambda_deriv} is positive, we arrive at \eqref{pdf_prob_eq_theorem} and finish the proof.
\end{proof}

Note that, for a given data set $\mathcal{D}$, the upper bound on $\delta_k$ in Theorem~\ref{th1} can be used to deterministically estimate the domain in which positive definiteness is fulfilled. A probabilistic computation can be made by means of the distribution of products and quotients of continuous, nonstandardized, independent Normal random variables \cite{rv_algebra}.

\subsection{Equilibria}

Despite the point-wise inclusion of $\bm{\gamma_0}=\bm{0}$ in the distribution \eqref{gauss_dist} being stochastic, we can provide an analytical guarantee for the preservation of the equilibrium at the origin.

\begin{theorem} \label{G_eq_theorem}
The L-GP-based potential mean estimate $\hat{V}(\bm{q})\equiv\mathrm{E}[V(\bm{q})|\bm{y},\bm{\gamma_0}]$ is guaranteed to have an equilibrium:
\begin{equation}
\hat{V}(\bm{0})=0\, , \quad \nabla_{\bm{q}} \hat{V}(\bm{0})=\bm{0} \, .
\end{equation}
\end{theorem}

\begin{proof}
Analogously to the derivation of \eqref{cond_mean_kin_init}, we now marginalize \eqref{gauss_dist} over $T$ instead of $V$, and condition on $\bm{y}$ as well as $\bm{\gamma_0}$. This leads to $\hat{V}(\bm{q})\equiv\mathrm{E}[V(\bm{q})|\bm{y},\bm{\gamma_0}]$ with
\begin{equation} 
\hat{V}(\bm{q}) = m_V(\bm{q}) + \left( \bm{k}_{\bm{y}V}^{\top} - \bm{k}_{\bm{0}V}^\top \bm{K}_{\bm{0}}^{-1} \bm{K}_{\bm{y0}}^\top \right)\bm{K}_D^{-1}\Delta\bm{y} \, . \label{cond_mean_pot_final}
\end{equation}
Based on the quadratic form of $\hat{U}$ proven in Lemma \ref{quad_form_lemma}, we set $U\equiv0$ w.l.o.g. and obtain $V=G$. Exploiting the structure $\eqref{k_G_form}$ of the gravitational kernel $k_G$ according to Assumption \ref{iso_quad_met_class_ass} then leads to $\bm{k}_{\bm{0}V}^\top(\bm{0})=[\sigma_G^2,\bm{0}^\top]$, block-diagonal $\bm{K_0} = \sigma_G^{2} \mathrm{diag}(1, \bm{\Lambda}_G)$ and thus $\bm{k}_{\bm{0}V}^\top(\bm{0})\bm{K_0}^{-1}=[1, \bm{0}^\top]$. Therefore,
\begin{equation} 
\bm{k}_{\bm{y}V}^\top(\bm{0}) = \bm{k}_{\bm{0}V}^\top(\bm{0})\bm{K_0}^{-1}\bm{K_{y0}}^\top \, . \label{eq_V_zero}
\end{equation}
Plugging \eqref{eq_V_zero} into \eqref{cond_mean_pot_final} and using $m_G(\bm{0})=0$ due to Assumption~\ref{ass_potential}, we follow $\hat{V}(\bm{0})= 0$ holds. Similarly applying the $\nabla$ operator to \eqref{cond_mean_pot_final}, we compute $\nabla_{\bm{q}}\bm{k}_{\bm{y}V}^\top(\bm{0}) = \nabla_{\bm{q}}\bm{k}_{\bm{0}V}^\top(\bm{0})\bm{K_0}^{-1}\bm{K_{y0}}^\top$. Thus we can conclude that $\nabla_{\bm{q}} \hat{V}(\bm{0})=\bm{0}$ holds, again for $\nabla_{\bm{q}} m_G(\bm{0})=\bm{0}$ as in Assumption~\ref{ass_potential}. %
\end{proof}

\begin{corollary} \label{corollary_eq_lagrange}
The estimative Lagrangian $\hat{L}\equiv\mathrm{E}[L|\bm{y},\bm{\gamma_0}]$ based on \eqref{gauss_dist} is analytically guaranteed to have a stationary point at $\bm{q}=\dot{\bm{q}}=\bm{0}$, i.e., $\nabla_{\bm{q},\dot{\bm{q}}} \hat{L}(\bm{0},\bm{0})=\bm{0}$, with $\hat{L}(\bm{0},\bm{0})=0$. 
\end{corollary}

\begin{proof}
Since $\hat{L}=\hat{T}-\hat{G}-\hat{U}$, this directly follows from the preservation of quadratic kinetic and elastic forms proven in Lemma \ref{quad_form_lemma} combined with the guaranteed gravitational equilibrium from Theorem \ref{G_eq_theorem}, i.e., $\hat{G}(\bm{0})=0$.
\end{proof}

\subsection{Energy Conservation}

Lastly, the L-GP is also equivalent to a dynamic system. We investigate this intuition in the following.

\begin{proposition}
Upon explicit inclusion of a test torque $\bm{\tau}(\bm{q},\dot{\bm{q}},\ddot{\bm{q}})$ in the form of \eqref{tau_vec_l_gp} in \eqref{gauss_dist}, the conditional expectation $\hat{\bm{\tau}}(\bm{q},\dot{\bm{q}},\ddot{\bm{q}})\equiv\mathrm{E}[\bm{\tau}(\bm{q},\dot{\bm{q}},\ddot{\bm{q}})|\bm{y},\bm{\gamma_0}]$ is given by 
\begin{equation}
\hat{\bm{\tau}}(\bm{q},\dot{\bm{q}},\ddot{\bm{q}}) = \hat{\bm{M}}(\bm{q})\ddot{\bm{q}} + \hat{\bm{C}}(\bm{q},\dot{\bm{q}})\dot{\bm{q}} + \hat{\bm{g}}(\bm{q}) \label{torque_gp_mean_est}
\end{equation}
with $\hat{\bm{C}}=\frac{1}{2}[\nabla_{\dot{\bm{q}}}\nabla_{\bm{q}}^\top \hat{T} + \dot{\hat{\bm{M}}} - \nabla_{\bm{q}}\nabla_{\dot{\bm{q}}}^\top \hat{T}]$, $\hat{\bm{M}}$ from \eqref{M_est} for $k=1$, and $\hat{\bm{g}}=\nabla_{\bm{q}}\hat{V}$ from \eqref{cond_mean_pot_final} and \eqref{M_est} for $k=0$. 
\end{proposition}

\begin{proof}
Follows as a direct consequence of the kernel construction via the Lagrangian-differential operator \eqref{lagrange_op} and the equivalent representation of \eqref{lagrange_nabla_eqs_of_mot} by \eqref{M_C_g_system}. Thus, we include \eqref{tau_vec_l_gp} in \eqref{gauss_dist} and get $$\mathrm{E}[\bm{\tau}(\bm{q},\dot{\bm{q}},\ddot{\bm{q}})|\bm{y},\bm{\gamma_0}]\equiv\bm{\mathcal{L}}_{\bm{q},\dot{\bm{q}}}\hat{L}(\bm{q},\dot{\bm{q}})=\hat{\bm{\tau}}(\bm{q},\dot{\bm{q}},\ddot{\bm{q}})$$ after marginalization and conditioning, proving the result.
\end{proof}

We now show that the equivalent L-GP dynamics \eqref{torque_gp_mean_est} preserve the physical property of energy conservation as a last result. For notational consistency, we use the (estimative) state vector $\hat{\bm{x}}^\top=[\bm{q}^\top,\dot{\bm{q}}^\top]$  in the following. 

\begin{corollary}\label{passivity_gp_corollary}
Consider the system $\dot{\hat{\bm{x}}}=\hat{\bm{\phi}}(\hat{\bm{x}},\hat{\bm{u}})$ defined $\forall \hat{\bm{x}}\in\bm{\Xi}$ in the compact domain $\bm{\Xi}\subseteq\bm{\mathcal{X}}\subseteq\mathbb{R}^{2N}$ with
\begin{equation} 
\bm{\Xi} = \big\{\bm{0}\big\} \cup \big\{ \hat{\bm{M}}(\hat{\bm{x}}_1) \succ \bm{0} \land \hat{V}(\hat{\bm{x}}_1) \geq -1/2\lambda_{N}(\hat{\bm{M}}) \lVert\hat{\bm{x}}_2\rVert^2 \big\} \nonumber
\end{equation}
spanned $\forall \hat{\bm{u}}\in\mathbb{R}^N$ by the GP \eqref{tau_vec_l_gp} with estimate \eqref{torque_gp_mean_est}, where
\begin{equation} 
\hat{\bm{\phi}}(\hat{\bm{x}},\hat{\bm{u}}) \!=\! \begin{bmatrix} \hat{\bm{x}}_2 \\ \hat{\bm{M}}^{-1}(\hat{\bm{x}}_1)\big(\hat{\bm{u}}-\hat{C}(\hat{\bm{x}}_1,\hat{\bm{x}}_2)\hat{\bm{x}}_2-\hat{\bm{g}}(\hat{\bm{x}}_1)\big) \end{bmatrix} \label{passive_dyn_sys} .
\end{equation}
The dynamic system \eqref{passive_dyn_sys} with state vector $\hat{\bm{x}}^\top=[\hat{\bm{x}}_1^\top,\hat{\bm{x}}_2^\top]$ and input $\hat{\bm{u}}=\hat{\bm{\tau}}$ is passive and, moreover, lossless, with respect to the energy storage function $\hat{E}=\hat{T}+\hat{V}$, where
\begin{equation} 
\hat{E}(\hat{\bm{x}}) =  1/2\hat{\bm{x}}_2^\top\hat{\bm{M}}(\hat{\bm{x}}_1)\hat{\bm{x}}_2 + \hat{V}(\hat{\bm{x}}_1) \, , \label{energy_storage_function}
\end{equation}
based on \eqref{quad_kinetic_mean}, \eqref{cond_mean_pot_final}, and the system output mapping $\hat{\bm{y}}=\hat{\bm{x}}_2$.
\end{corollary}

\begin{proof} From the definition of $\bm{\Xi}$, we can follow that \eqref{passive_dyn_sys} is locally Lipschitz in the domain $\bm{\Xi}\times\mathbb{R}^N$. Also, the storage function \eqref{energy_storage_function} is positive semidefinite, since $\hat{E}(\bm{0})=0$ due to Corollary \ref{corollary_eq_lagrange} and $\lambda_{N}(\hat{\bm{M}})\lVert\hat{\bm{x}}_2\rVert^2\leq\hat{\bm{x}}_2^\top\hat{\bm{M}}\hat{\bm{x}}_2$ for $\hat{\bm{M}} \succ \bm{0}$ ensures $\hat{E}(\hat{\bm{x}})\geq0$, $\forall \bm{x}\in\bm{\Xi}$. The output function $\hat{\bm{y}}=\hat{\bm{x}}_2$ is continuous and for the dynamics $\hat{\bm{\phi}}(\bm{0},\bm{0})=\bm{0}$ holds. Taking the time derivative of \eqref{energy_storage_function} along the trajectories \eqref{passive_dyn_sys}, we obtain the power fed into the system by input $\hat{\bm{u}}$ computing
\begin{align}
\dot{\hat{E}} &= \hat{\bm{x}}_2^\top\left( \hat{\bm{M}}(\hat{\bm{x}}_1)\dot{\hat{\bm{x}}}_2 + \frac{1}{2}\dot{\hat{\bm{M}}}(\hat{\bm{x}}_1)\hat{\bm{x}}_2 + \nabla_{\hat{\bm{x}}_1}\hat{V}(\hat{\bm{x}}_1) \right) \nonumber \\
&= \hat{\bm{x}}_2^\top\left( \hat{\bm{M}}(\hat{\bm{x}}_1)\dot{\hat{\bm{x}}}_2 + \hat{C}(\hat{\bm{x}}_1,\hat{\bm{x}}_2)\hat{\bm{x}}_2 + \hat{\bm{g}}(\hat{\bm{x}}_1) \right)  \nonumber \\
&= \hat{\bm{y}}^\top\hat{\bm{u}} \nonumber \, ,
\end{align}
where we have exploited the symmetry of $\hat{\bm{M}}=\hat{\bm{M}}^\top$, the identity $\dot{\hat{\bm{M}}}=\hat{\bm{C}}+\hat{\bm{C}}^\top$ and the skew-symmetry of $\dot{\hat{\bm{M}}} - 2\hat{\bm{C}}^\top$. Thus, according to \cite{khalil}, the system \eqref{passive_dyn_sys} is lossless.
\end{proof}

Note that from the passivity of system \eqref{passive_dyn_sys} we can also directly follow stability of the origin of the unforced system $\dot{\hat{\bm{x}}}=\hat{\bm{\phi}}(\hat{\bm{x}},\bm{0})$ by taking \eqref{energy_storage_function} as a Lyapunov function \cite{khalil}.

\section{NUMERICAL ILLUSTRATIONS} \label{sim_sec}

\addtolength{\textheight}{-1.28725cm}

In this section, we now demonstrate the efficacy of our approach and validate the derived theoretical results, choosing a simple example for the sake of easy interpretation. 

\subsection{Setup}

We benchmark our method on the two-link robotic manipulator from \cite[p.~164]{murray_amitrm}. Gravity $g=10$ ms$^{-2}$ acts along the positive $x$-axis, leading to an equilibrium at the origin $\bm{q}=\bm{0}$. The links have unit masses $m_n=1$~kg and unit lengths $l_n=1$~m for $n\in\{1,2\}$. Estimates are available but erroneous: $\hat{m}_n=(1+\Delta_n)m_n$, $\hat{l}_n=(1+\Delta_n)l_n$, where $\Delta_n=(-1)^n/2$.

For all of the numerical simulations, we use only 25 training pairs equally distanced on the domain $\bm{q}\in[-1,1]^2$ with fixed $\dot{q}_n=(-1)^n$ and $\ddot{\bm{q}}=\bm{1}$. The torque measurement noise has covariance $\bm{\Sigma}_{\bm{\varepsilon}_i}=\sigma_{\bm{\varepsilon}}^2\bm{I}$, $\sigma_{\bm{\varepsilon}}=0.1$ Nm, while the differential process noise corrupting the accelerations has covariance $\bm{\Sigma}_{\bm{\alpha}_i}=\sigma_{\alpha}^2\bm{I}$, $\sigma_{\alpha}=\pi/180$ $\text{rad}/\text{s}^{2}$. Positions and velocities are kept noisefree, enabling explicit numerical integration with ode45. We reduce the kinetic mass inertia hypermetrics $\bm{\Lambda}_{knm}=\bm{\Lambda}_{k}$, $\forall n,m\in\{1,2\}$, to the constant Euclidian form $\bm{\Lambda}_{k}^{-1}=\sigma_{\bm{d}_{T}}^{2}\bm{I}$, where $k=1$, and fix $\sigma_{\bm{d}_{T}}=10^{2}$. The same is done for the gravitational distance covariance keeping $\bm{\Sigma}_{\bm{d}_G}^{1/2}=\mathrm{diag}([1.6,2.7])$. The remaining radial gravitational hypervariance $\sigma_G^2$ along with the Cholesky-kinetic upper-triangular hypercovariance $\bm{\Sigma}_{f_k}$, where again $k=1$, are optimized via the log-likelihood.

\subsection{Results}

As a first experiment, we demonstrate the closed-loop applicability of the L-GP model. For this, we compare a standard PD tracking controller \cite[p.~194]{murray_amitrm} with its L-GP-based version as shown in Fig.~\ref{fig_1_tracking_control}. Starting from the equilibrium, the controllers with gains $\bm{K}_p=\bm{K}_d=10\bm{I}$ have the task of following a sinusoidal reference trajectory $\bm{q}_d(t)=a_{d}\sin t\bm{1}$ with amplitude $a_d=\pi/2$. The simulation shows that the L-GP-based version considerably improves the accuracy, demonstrating reliable performance despite the suboptimal distribution of the data points, independent from the reference trajectory, requiring substantial extrapolation.

\begin{figure}[] 
\centerline{\includegraphics[width=\columnwidth]{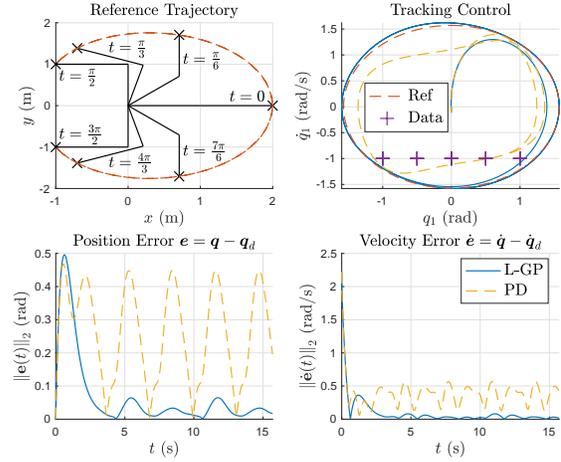}}
\caption{Tracking performance of the standard PD controller (dashed yellow line) and its L-GP augmented version (solid blue line). The top left plot shows the reference trajectory (dashed red line) w.r.t.\ the end-effector in its cartesian work space, the top right the closed-loop behavior in the state space of the first joint along with the training data points (purple crosses). The bottom two subfigures indicate the Euclidian norms of position and velocity tracking errors over time.}
\label{fig_1_tracking_control}
\end{figure}

\begin{figure}[] 
\centerline{\includegraphics[width=\columnwidth]{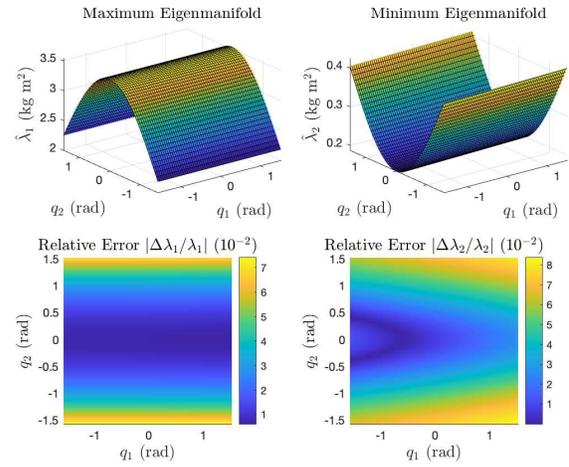}}
\caption{Eigenmanifolds of the generalized inertia matrix of the two-link robot over the domain $\bm{q}\in[-\pi/2,\pi/2]^2$. The top left and right surface plots show the maximum and minimum estimative eigenvalue manifolds of the L-GP-based estimate $\hat{\bm{M}}(\bm{q})$, respectively, while the bottom two subfigures are heat maps indicating the relative approximation errors in percent.}
\label{fig_2_eigenmanifolds}
\end{figure} 

Next, we validate the physical consistency of the L-GP. Therefore, using the same parametrization as in the previous experiment, we evaluate the mass inertia estimate over the joint domain $\bm{q}\in[-a_d, a_d]^2$ and investigate its eigenmanifolds, as shown in Fig.~\ref{fig_2_eigenmanifolds}. Clearly, the L-GP accurately approximates the positive definite function space despite only being a subcomponent of the input-output relation. The consistency of the gravitational potential estimate with equilibrium is validated in Fig.~\ref{fig_3_energy}, along with the conservatism of the equivalent L-GP-based dynamics simulated for different initial conditions $\hat{\bm{x}}(0)=[a_0\bm{1}^\top,\bm{0}^\top]^\top$ with displacement amplitudes $a_0=0.1$ (red line), $a_0=0.5$ (yellow line) and $a_0=1$ (purple line). The signed potential and energy approximation errors show their properties of local positive definiteness and passivity, respectively.

\begin{figure}[] 
\centerline{\includegraphics[width=\columnwidth]{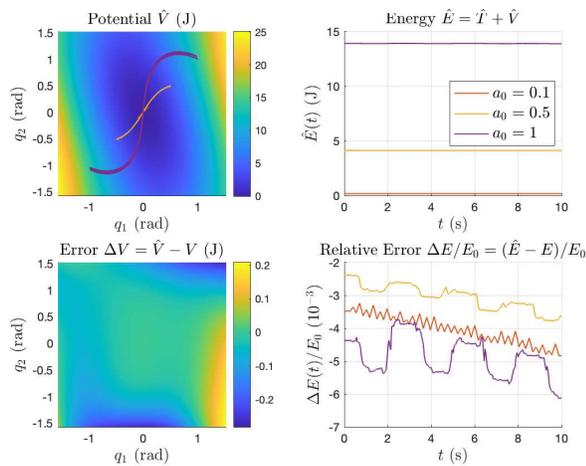}}
\caption{Potential energy estimate and free trajectories of the L-GP equivalent dynamic system (top left plot) with initial conditions $\bm{q}(0)=a_0\bm{1}$ and $\dot{\bm{q}}(0)=\bm{0}$ for $a_0=0.1$ (red line), $a_0=0.5$ (yellow line) and $a_0=1$ (purple line). The constant time evolution of the energy estimates are shown in the top right plot for the same trajectories. The bottom left plot visualizes the signed potential approximation error, the bottom right one the signed relative energy error in permille.}
\label{fig_3_energy}
\end{figure}

\section{CONCLUSION}\label{conclus_sec}

This paper presented an approach for the identification of uncertain Lagrangian systems exploiting prior physical knowledge by kernel construction. Physical consistency of the data-driven method in terms of guarantuees for the fulfillment of certain properties was proven rigorously and validated in numerical simulation, along with the effective applicability of the model to an exemplary tracking control problem. Future research will focus on the extension of the approach to handle dissipative and time-variant systems. Another promising direction is the application of the Lagrangian variance estimate for uncertainty quantification. 

\section{ACKNOWLEDGMENTS}

This work was supported by the Consolidator Grant "Safe data-driven control for human-centric systems" (CO-MAN) of the European Research Council (ERC) under grant agreement ID 864686, and by the Horizon 2020 Research and Innovation Action project "Rehabilitation based on Hybrid neuroprosthesis" (ReHyb) of the European Union (EU) under grant agreement number 871767. The authors gratefully acknowledge the thoughtful comments from L. Evangelisti, Institute of System Dynamics and Control, German Aerospace Center (DLR), and the fruitful discussions with A. Lederer, Chair of Information-oriented Control, Technical University of Munich.
%

\bibliographystyle{ieeetr}
\bibliography{mybibfile}

\begin{thebibliography}{10}

\bibitem{perf_guaraentees_gp}
A.~Gahlawat, P.~Zhao, A.~Patterson, N.~Hovakimyan, and E.~Theodorou, ``L1-gp:
  L1 adaptive control with bayesian learning,'' in {\em Proceedings of the 2nd
  Conference on Learning for Dynamics and Control}, vol.~120 of {\em
  Proceedings of Machine Learning Research}, pp.~826--837, PMLR, 10--11 Jun
  2020.

\bibitem{berkenkamp}
F.~Berkenkamp, R.~Moriconi, A.~P. Schoellig, and A.~Krause, ``Safe learning of
  regions of attraction for uncertain, nonlinear systems with gaussian
  processes,'' in {\em 2016 IEEE 55th Conference on Decision and Control
  (CDC)}, pp.~4661--4666, 2016.

\bibitem{data_eff_gps}
M.~P. Deisenroth, D.~Fox, and C.~E. Rasmussen, ``Gaussian processes for
  data-efficient learning in robotics and control,'' {\em IEEE Transactions on
  Pattern Analysis and Machine Intelligence}, vol.~37, no.~2, pp.~408--423,
  2015.

\bibitem{sys_id_ljung}
L.~Ljung, {\em System Identification: Theory for the User}.
\newblock PTR Prentice Hall Information and System Sciences, Pearson, 2nd~ed.,
  1998.

\bibitem{param_id_leboutet}
Q.~Leboutet, J.~Roux, A.~Janot, J.~R. Guadarrama-Olvera, and G.~Cheng,
  ``Inertial parameter identification in robotics: A survey,'' {\em Applied
  Sciences}, vol.~11, no.~9, 2021.

\bibitem{soft_robotics_com}
A.~A. Amiri~Moghadam, K.~Torabi, A.~Kaynak, M.~N.~H. Zainal~Alam, A.~Kouzani,
  and B.~Mosadegh, ``Control-oriented modeling of a polymeric soft robot,''
  {\em Soft Robotics}, vol.~3, no.~2, pp.~82--97, 2016.

\bibitem{day_couplings_air}
R.~E. Day, ``Coupling dynamics in aircraft: A historical perspective,'' special
  publication, NASA Dryden Flight Research Center, Edwards, CA, 1997.

\bibitem{hydro}
T.~I. Fossen and O.-E. Fjellstad, ``Nonlinear modelling of marine vehicles in 6
  degrees of freedom,'' {\em Mathematical Modelling of Systems}, vol.~1, no.~1,
  pp.~17--27, 1995.

\bibitem{symp_GPR_ham_sys}
K.~Rath, C.~G. Albert, B.~Bischl, and U.~von Toussaint, ``Symplectic gaussian
  process regression of maps in hamiltonian systems,'' {\em Chaos: An
  Interdisciplinary Journal of Nonlinear Science}, vol.~31, no.~5, p.~053121,
  2021.

\bibitem{cheng_lagrangian_rkhs}
C.-A. Cheng and H.-P. Huang, ``Learn the lagrangian: A vector-valued rkhs
  approach to identifying lagrangian systems,'' {\em IEEE Transactions on
  Cybernetics}, vol.~46, no.~12, pp.~3247--3258, 2016.

\bibitem{thomas_tracking_control_of_el_systems}
T.~Beckers, D.~Kuli{\'c}, and S.~Hirche, ``Stable gaussian process based
  tracking control of euler--lagrange systems,'' {\em Automatica}, vol.~103,
  pp.~390--397, 2019.

\bibitem{murray_amitrm}
R.~M. Murray, Z.~Li, and S.~S. Sastry, {\em A Mathematical Introduction to
  Robotic Manipulation}.
\newblock CRC Press, 1994.

\bibitem{nips_gp_input_noise}
A.~Mchutchon and C.~Rasmussen, ``Gaussian process training with input noise,''
  in {\em Advances in Neural Information Processing Systems}, vol.~24, Curran
  Associates, Inc., 2011.

\bibitem{johnson_input_noise_gp}
J.~E. Johnson, V.~Laparra, and G.~Camps-Valls, ``Accounting for input noise in
  gaussian process parameter retrieval,'' {\em IEEE Geoscience and Remote
  Sensing Letters}, vol.~17, no.~3, pp.~391--395, 2020.

\bibitem{rasmussen_gp_mit_book}
C.~E. Rasmussen and C.~K.~I. Williams, {\em Gaussian Processes for Machine
  Learning}.
\newblock The MIT Press, 2006.

\bibitem{raissi_gp_diff_eqs}
M.~Raissi, P.~Perdikaris, and G.~E. Karniadakis, ``Inferring solutions of
  differential equations using noisy multi-fidelity data,'' {\em Journal of
  Computational Physics}, vol.~335, pp.~736--746, 2017.

\bibitem{nips_deriv_obs_solak}
E.~Solak, R.~Murray-smith, W.~Leithead, D.~Leith, and C.~Rasmussen,
  ``Derivative observations in gaussian process models of dynamic systems,'' in
  {\em Advances in Neural Information Processing Systems}, vol.~15, MIT Press,
  2002.

\bibitem{md_gps_alvarez}
M.~A. {\'A}lvarez, L.~Rosasco, and N.~D. Lawrence, ``Kernels for vector-valued
  functions: A review,'' {\em Foundations and Trends{\textregistered} in
  Machine Learning}, vol.~4, no.~3, pp.~195--266, 2012.

\bibitem{kf_gp_rel}
S.~Reece and S.~Roberts, ``An introduction to gaussian processes for the kalman
  filter expert,'' in {\em 2010 13th International Conference on Information
  Fusion}, pp.~1--9, 2010.

\bibitem{lovlelock_metrics}
D.~Lovelock and H.~Rund, {\em Tensors, Differential Forms, and Variational
  Principles}.
\newblock Dover Publications, 1989.

\bibitem{knutson2000honeycombs}
A.~Knutson and T.~Tao, ``Honeycombs and sums of hermitian matrices,'' {\em
  Notices Of The American Mathematical Society}, vol.~48, no.~2, pp.~175--186,
  2001.

\bibitem{bhatia_ma}
R.~Bhatia, {\em Matrix Analysis}.
\newblock Graduate Texts in Mathematics, Springer, New York, NY, 1~ed., 1997.

\bibitem{rv_algebra}
M.~D. Springer, {\em The Algebra of Random Variables}.
\newblock John Wiley \& Sons, Inc., 1979.

\bibitem{khalil}
H.~K. Khalil, {\em Nonlinear Systems}.
\newblock Prentice Hall, third~ed., 2002.

\end{thebibliography}

\end{document}